\documentclass[runningheads,a4paper]{llncs}
\usepackage[bookmarks=false]{hyperref}
\usepackage{cite}

\usepackage[utf8]{inputenc}
\usepackage{amsmath}
\usepackage{amsfonts}
\usepackage{amscd}
\usepackage{url}
\usepackage{enumerate}

\newlength{\oldfb}
\newcommand{\boxx}[1]
{\setlength{\oldfb}{\fboxrule}\setlength{\fboxrule}{2pt}\framebox{\parbox{\dimexpr\linewidth-2\fboxsep-2\fboxrule}{#1}}\setlength{\fboxrule}{\oldfb}\\}

\usepackage{mathtools}
\newcommand{\defeq}{\vcentcolon=}

\renewcommand{\vec}[1]{\overline{#1}}

\spnewtheorem{thm}{Theorem}[section]{\bfseries}{\itshape}
\spnewtheorem{pro}[thm]{Proposition}{\bfseries}{\itshape}
\spnewtheorem{lem}[thm]{Lemma}{\bfseries}{\itshape}
\spnewtheorem{defi}[thm]{Definition}{\bfseries}{\itshape}
\spnewtheorem{rem}[thm]{Remark}{\bfseries}{\itshape}

\newcommand{\V}[1]{\vec{#1}}
\DeclareMathOperator{\tr}{tr}

\DeclareMathOperator{\Id}{Id}

\newcommand{\R}{\mathbb{R}}

\newcommand{\Th}{T_h}
\newcommand{\OTh}{{\cal{T}}_h}
\newcommand{\Oo}{\mathcal{O}}

\newcommand{\Tg}{T_g}
\newcommand{\OTg}{{\cal{T}}_g}

\newcommand{\OTH}{\mathcal{T}_H}

\def\calD{\mathcal{D}}
\def\calT{\mathcal{T}}

\newcommand{\La}{\Lambda}
\newcommand{\Lai}{\Lambda^{(i)}}
\newcommand{\la}{\lambda}
\newcommand{\f}{f}

\def \R {{\mathbb R}} 
\def \Tg {\mathcal{T}_g}
\def \V {V}
\def \Vk {V^k}

\title{Inability of spatial transformations of CNN feature maps to support invariant recognition}
\titlerunning{~}
\authorrunning{~}

\author{Ylva Jansson \and Maksim Maydanskiy \and Lukas Finnveden \and Tony Lindeberg}

\institute{Computational Brain Science Lab \\ Division of Computational Science and Technology\\ 
	KTH Royal Institute of Technology\\
	Stockholm, Sweden 
}

\begin{document}
	
\maketitle

\begin{abstract}
A large number of deep learning architectures use spatial transformations of CNN feature maps or filters to better deal with variability in object appearance caused by natural image transformations.
In this paper, we prove that spatial transformations of CNN \emph{feature maps} cannot align the feature maps of a transformed image to match those of its original, for general affine transformations, unless the extracted features are \emph{themselves invariant}.
Our proof is based on elementary analysis for both the single- and multi-layer network case. 
  The results imply that methods based on spatial transformations of CNN feature maps or filters cannot replace image alignment of the input and \emph{cannot enable invariant recognition} for general affine transformations, specifically not for scaling transformations or shear transformations. 
For rotations and reflections, spatially transforming feature maps or filters can enable invariance but only for networks with learnt or hardcoded rotation- or reflection-invariant features.

\end{abstract}

\section{Introduction}
Convolutional neural networks (CNNs) that are \emph{invariant} to certain groups of image transformations have fewer parameters, can learn from smaller datasets and enable \emph{generalization outside the training distribution}. 
A number of current methods use spatial transformations of CNN feature maps or filters to enhance the ability of CNNs to handle different types of image transformations \cite{ChoGwaSavSil-NIPS2016,LiCheCaiDav-arXiv2017,KimLinJeoMin-NIPS2018,zheng2018pedestrian,HeZhaXia-ECCV2014,yuarXiv2015,DaiQiXio-arXiv2017,JadSimZisKav-NIPS2015}.
For example, \emph{spatial transformer networks} (STNs) \cite{JadSimZisKav-NIPS2015} were designed to enable CNNs to learn invariance to image transformations by transforming \emph{CNN feature maps} as well as input images. 
Clearly, if a network learns to align transformed input images to a common pose, this can enable invariant recognition. The original work \cite{JadSimZisKav-NIPS2015}, however, simultaneously claims the ability of STNs to learn invariance from data and that the spatial transformer layers (STs) can be inserted into the network ``anywhere" (i.e. at any depth). 
There is no mention of whether the key motivation for the framework - the ability to learn invariance - is still supported when transforming feature maps deeper in the network. 

This seems to have left some confusion about whether spatially transforming CNN feature maps can support invariant recognition. A number of subsequent works advocate image alignment by \emph{transforming feature maps} \cite{ChoGwaSavSil-NIPS2016,LiCheCaiDav-arXiv2017,KimLinJeoMin-NIPS2018,zheng2018pedestrian}, including e.g. pose alignment of pedestrians \cite{zheng2018pedestrian} and use of a spatial transformer to mimic the kind of patch normalization done in SIFT \cite{ChoGwaSavSil-NIPS2016}. 
Other commonly used methods that are based on transforming CNN feature maps or filters are spatial pyramid pooling \cite{HeZhaXia-ECCV2014}, dilated convolutions \cite{yuarXiv2015} and deformable convolutions \cite{DaiQiXio-arXiv2017}. Such methods are often motivated by the need for CNNs to better deal with variability in object pose. 
There is, however, no discussion about the difference between pose normalizing the input image and spatially transforming feature maps, or the implications this choice has for the ability to achieve e.g. affine or scale invariance \cite{HeZhaXia-ECCV2014,yuarXiv2015,DaiQiXio-arXiv2017,JadSimZisKav-NIPS2015}.

Here, we elucidate under what conditions it is possible to achieve invariance to affine image transformations by means of \emph{purely spatial transformations} of CNN feature maps. These conditions turn out to be very restrictive, implying network filters or features that are \emph{already invariant} to the relevant image transformations. 
This implies that spatial transformations of CNN feature maps
\emph{cannot}, in general, align the feature maps of a transformed image with those of an original and thus not enable affine-invariant recognition. The exception is translations, where the translation covariance of CNNs does imply that translations and feature extraction do commute. 

We do not claim much mathematical novelty of these facts, which are in some sense intuitive, and, in the single-layer case, have some parallels with the work in \cite{cohen2016group} and \cite{CohGeiWei-NIPS2019}. 
Our contribution is to present an alternative proof based on elementary analysis for the special case of purely spatial transformations of CNN feature maps (as opposed to more general transformations that might mix information between the different feature channels). Since we only consider spatial transformations, we can give a more direct proof. 
We also provide an analysis of the general multi-layer case, without relying on any covariance assumptions about the  individual layers. 

 Our results have straightforward implications for STNs and other methods that perform spatial transformations of CNN feature maps or filters. 
An experimental evaluation of the practical consequences of our result in the context of \emph{spatial transformer networks}, together with a short intuitive version of the proof presented here, has been presented in \cite{FinJanLin-arXiv2020}. 

\section{Preliminaries}

\subsection{Images and image transformations}
We work with a continuous model of the image space. We consider both an \textbf{image} $\f$ and a convolutional \textbf{filter} $\la$ to be a map from $\R^N$ to $\R$. We use notation  $V$ for the function space to which the images $f$ belong, and $\Vk$ for the space of maps that have each of their $k$ components in $V$. We are somewhat lax about specifically what class of functions $\lambda$ and $f$ should belong to. We need that the convolution operator
\begin{equation}
\La_\la f(x) \defeq  (f\star \la)(x)=\int_{\R^N} f(y)\la(x-y) dy=\int_{\R^N} \la(y) f(x-y) dy
\end{equation}
is defined and has output that lies in the same space, and that applying a Lipshitz continuous point-wise non-linearity $\sigma$ to an image also produces an image in the same space. This will hold for example if $\la$ are integrable and compactly supported (we'll write $\la \in L^1_{comp}$) and the images $f$ are locally integrable ($f\in L^1_{loc}$). Hence, when necessary we will assume $V$ to be the space of locally integrable functions (with the corresponding $L^1_{loc}$ topology). To avoid possible confusion, we denote the zero function by $0$ and the point $0 \in \R^N$ by $\overline{0}$.

\subsection{Continuous model of a CNN}\label{sec:CNN}

Let $\La:  V \to V^{M_k}$ denote a \emph{continuous CNN} with $k$ layers and $M_k$ feature channels in the final layer 
and let $\theta^{(i)}$ represent the transformation between layers $i-1$ and $i$ such that
\begin{equation}
(\Lambda f)_c(x) = (\theta^{(k)} \theta^{(k-1)} \cdots \theta^{(2)} \theta^{(1)} f)_c(x)
\label{eq:phi-def},
\end{equation}
where $c \in \{1,2, \dots M_k\}$ denotes the feature channel. Let further $\Lambda^{(i)} f$ refer to the output from layer $i$ (with $M_i$ feature channels and $\La^{(0)} f = f$)
\begin{align}
\Lai f &= \theta^{(i)} \theta^{(i-1)} \cdots \theta^{(2)} \theta^{(1)} f. 
\label{eq:phi_i-def}
\end{align} 
We model the transformation $\theta^{(i)}$ between two adjancent layers $\Lambda^{(i-1)}f$ and $\Lambda^{(i)}f$ as a convolution followed by the addition of a bias term $b_{i,c} \in \R$ and the application of a pointwise non-linearity $\sigma_i:\R \to \R$:
\begin{multline}
(\Lai f)_c (x)=\sigma_i \left( \sum_{m=1}^{M_{i-1}} \int_{y \in \R^N } (\Lambda^{(i-1)}f)_m (x-y)\, \lambda^{(i)}_{m,c}(y) \, dy + b_{i,c}
\right),
\label{eq:CNN}
\end{multline}
where $\lambda^{(i)}_{m,c} \in L^1_{comp}$
denotes the convolution kernel that propagates information from feature channel $m$ in layer $i-1$ to output feature channel $c$ in layer $i$. 
A final fully connected classification layer with compact support can also be modelled as a convolution combined with a non-linearity $\sigma_k$ that represents \emph{a softmax operation} over the feature channels. 

We note that since a convolution with $\la \in L^1_{comp}$ is a continuous operator from $V$ to $V$ (recall that we are using $L^1_{loc}$ topology, so the continuity follows from the  $L^1$ norm inequality for convolutions, see \cite{stein2009real}, Chapter 2, Exercise 21 d), we conclude that when the $\sigma_i$s are Lipschitz continuous functions the resulting $\La:  V \to V^{M_k}$ is a continuous operator.

\subsection{Transformations of images and feature maps}
We will consider the group of \emph{affine image transformations}, which here correspond to a collection of linear maps\footnote{We are thus not interested in translations.} $\Th: \R^N\to \R^N$. 
For each such map, we have a corresponding operator $\OTh^k:V^k\to V^k$, defined by the ``contragradient" representation, that is by precomposing with $T_h^{-1}$, as follows:

\begin{defi}\label{def:op-Th}
We define $\OTh^k: V^k \to V^k$, first for input images, by setting

\begin{equation}\label{eg:Th-def}
(\OTh^1 f) (x)= f(T_h^{-1} x)
\end{equation} 
and then on feature maps as
\begin{equation}\label{eg:Th-def-k}
(\OTh^{k} \La f)_c (x)= (\La f)_c(T_h^{-1} x),
\end{equation} 
where $k$ denotes the number of feature channels.
\end{defi}

Note how this definition implies purely spatial transformations of feature maps. Although the $\OTh^k$'s are, technically, different operators for different values of $k$ we often refer to all these operators as $\OTh$ to simplify the notation.

\begin{defi}\label{def:op-D}
	We define the \textbf{translation operator} $\calD_{\delta}$, with  $\delta \in \R^N$ for input images by 
	\begin{equation}
	(\calD_\delta f) (x) = f(x - \delta)
	\end{equation}
	and then for feature maps by
	\begin{equation}
	(\calD_\delta^k \La f) (x) = (\La f)(x - \delta).
	\end{equation}
\end{defi} 

We will again use single notation $\calD_\delta$ for all operators $\calD_\delta^k:V^k\to V^k$.

\subsection{Invariance and covariance}
Consider a general (possibly non-linear) feature extractor $\Lambda: \V \to \Vk$ such as e.g. the continuous analog of a CNN described in Section \ref{sec:CNN}.

\begin{defi}\label{def:covariance}
	We define an operator $\Lambda$ to be \textbf{covariant} to an operator $\Oo$ if there exists an input independent operator $\Oo'$ such that we can express a communative relation over $\Lambda$ of the form (see also Figure \ref{fig:commdia})
	\begin{equation}
	\Lambda \Oo  f = \Oo' \Lambda f.
	\label{eq:covariance1}
	\end{equation}
	If such an operator exists and is in addition invertible, then  it is possible to ``undo" the action of $\Oo$ after feature extraction. 
	(In the invariant neural networks literature, covariance is also often referred to as equivariance.) 

\end{defi}

\begin{figure}[h]
	\[
	\begin{CD}
	{\Lambda} \, f @>{\Oo'}>> \Lambda \Oo f \\
	\Big\uparrow\vcenter{\rlap{$\scriptstyle{{\Lambda}}$}} & & \Big\uparrow\vcenter{\rlap{$\scriptstyle{{\Lambda}}$}} \\
	f @>{\Oo}>> {\Oo f} 
	\end{CD}
	\]
	\caption{Commutative diagram for a covariant feature extractor $\Lambda$.}
	\label{fig:commdia}
\end{figure}

We here consider operators $\OTh^k$ corresponding to 
affine transformations of the spatial image domain that do not mix information between the feature channels (Definition \ref{def:op-Th}), which leads us to study (restricted) covariance relations of the form: 	\begin{equation}
	\Lambda \calT_h f = (\OTg^{k})^{-1} \Lambda f.
	\label{eq:covariance2}
	\end{equation}
We ask the question if and under what  conditions such (restricted) covariance relations exist for CNNs.

\begin{defi}\label{def:translation-covariance}
	We define an operator $\La$ to be \textbf{translation covariant} if for every $\delta$ we have
	\begin{align}\label{eq:tr-covar} 
	\La \calD_{\delta}=\calD_{\delta} \La.
	\end{align}
\end{defi}

\begin{defi}\label{def:invariance}
	We define an operator $\Lambda$ to be \textbf{invariant} to an operator $\OTh$ if the feature representation of a transformed image is \emph{equal to} the feature representation of the original image 
	\begin{equation}
	\Lambda \calT_h f = \Lambda f 
	\label{eq:invariance}
	\end{equation}
	for all $f \in V$. If this is true for all $h$ in a transformation group $H$, we say that $\Lambda$ is invariant to $H$.
\end{defi}

\begin{lem}\label{lemma:conv-trans-covar} The convolution operator is translation covariant
	\begin{equation}\calD_{\delta}\La_\la  =\La_\la  \calD_{\delta} =\La_{\calD_{\delta}\lambda}.
	\end{equation}
\end{lem}
The proof is given in Appendix \ref{app:single-layer-covariance}.

\begin{pro}\label{prop:CNN-covar} A  CNN as defined in Section \ref{sec:CNN} is a translation-covariant operator.
\end{pro}

\begin{proof}[Sketch]
	Since each convolution operation is translation covariant by Lemma \ref{lemma:conv-trans-covar} and the nonlinearities act on the values returned as output from the convolutions, all the operators $\Lambda^{(i)}$ are  translation covariant. Formal proof is by induction on $i$ (see Appendix \ref{app:prop-CNN-proof}).
\end{proof}

\begin{lem} \label{lem:conv-lin}
	Translation and general linear operators (c.f. (\ref{eg:Th-def})) have the following commutation relation:
	
	\begin{equation} \label{eq:commutator1} \OTh \calD_\delta  = \calD_{(\Th \delta)} \OTh  
	\end{equation}
	or equivalently
	\begin{equation} \label{eq:commutator2} \calD_\delta \OTh = \OTh \calD_{(\Th^{-1} \delta)}.  
	\end{equation}
	
\end{lem}

\begin{proof} Applying both sides to $f$ we compute 
	\begin{align}
	(\OTh \calD_\delta f)(x)=(\calD_\delta f)(T_h^{-1}(x))=f(T^{-1}_h(x)-\delta),
	\end{align}
	\begin{align}
	(\calD_{(\Th \delta)} \OTh f)(x) = (\OTh f)(x-\Th \delta)=f(T^{-1}_h(x-\Th \delta))=f(T^{-1}_h(x)-\delta).
	\end{align}
\end{proof}

\begin{figure}[htbp]
	\begin{center}
		\includegraphics[width=0.5\textwidth]{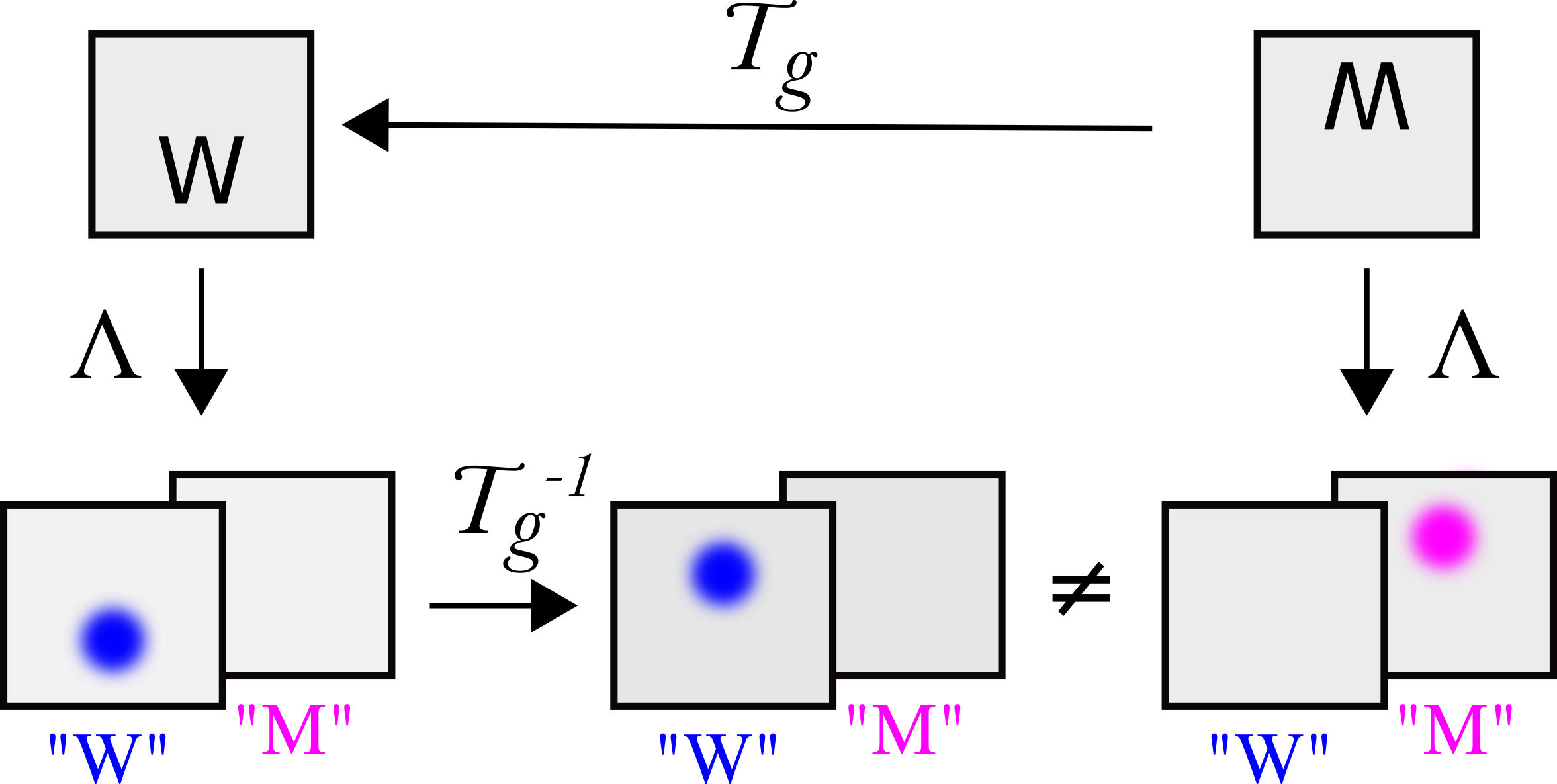}
	\end{center}
	\caption{\emph{An inverse spatial transformation of a \emph{CNN feature map} cannot, in general, align the feature maps of a transformed image with those of its original}. Here, the network $\La$ has two feature channels ``W'' and ``M'', and $T_g$ corresponds to a 180$^\circ$ rotation. Since different \emph{feature channels} respond to the rotated image as compared to the original image, it is not possible to align the respective feature maps with a spatial rotation. In fact, \emph{spatially transforming feature maps} can, in most cases, not eliminate differences related to object pose and can thus not enable invariant recognition.}
	\label{fig:tiny-proof}
\end{figure}


\begin{figure}[htbp]
	\begin{center}
		\includegraphics[width=0.7\textwidth]{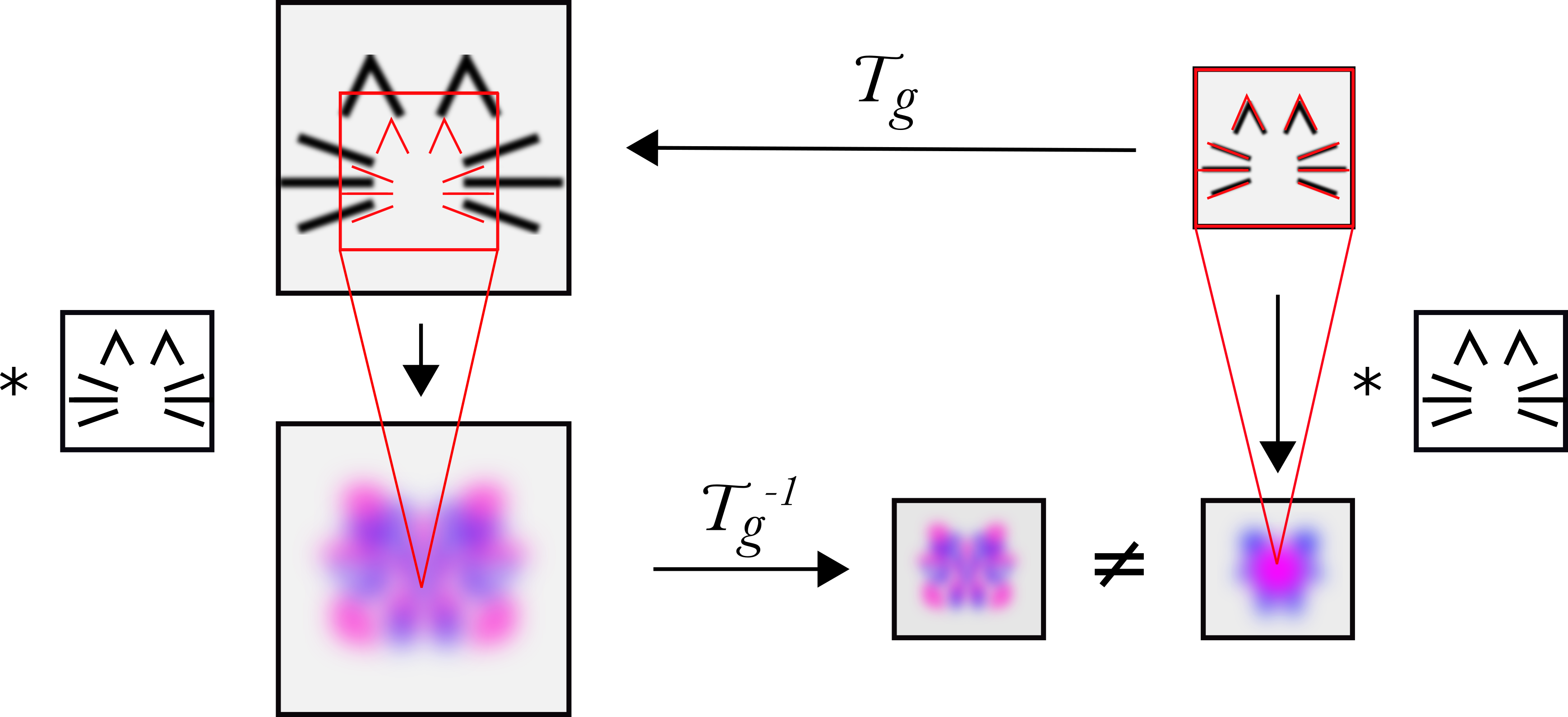}
	\end{center}
	\caption{\emph{For any transformation that includes a scaling component, the field of view of a feature extractor with respect to an object will differ between an original and rescaled image.}
		Consider e.g. a simple linear model that performs template matching with a single filter. When applied to the original image, the filter matches the size of the object that it has been trained to recognize and thus responds strongly. When applied to a rescaled image, the filter never covers the full object of interest, and thus the response cannot be guaranteed to take even \emph{the same set of values} for a rescaled image and its original.
	}
	\label{fig:tiny-proof-scale}
\end{figure}

\section{Intuition and outline of proof}
A spatial transformation of \emph{an input image} can clearly support invariant recognition by applying the inverse transformation to a transformed input:
\begin{equation} 
\La\, \OTh^{-1} \OTh f = \La f.
\end{equation}
The key question is whether it is possible to in a similar way undo a transformation of an input image \emph{after feature extraction}. Is there a spatial transformation $\Tg^k$ dependent on $\OTh$ such that 
at a certain depth in the network
\begin{equation}
\OTg^k \La^{(i)} \OTh f  \stackrel{?}{=}  \La^{(i)} f 
\label{eq:feature_alignment}
\end{equation}
holds for all $f$.  Note that this would imply that $\La^{(i)}$ is (restricted) covariant to $\OTh$.
Remember that, since we consider spatial transformations of feature maps, the same transformation is applied in each feature channel
\begin{equation}(\Tg^k \La^{(i)} \OTh f)_c(x) = (\La^{(i)} \OTh f)_c(T_g^{-1} x).
\end{equation}
 Clearly, if (\ref{eq:feature_alignment}) holds then transformations of feature maps could enable invariant recognition in a similar way as for input images. The feature maps of transformed images could be aligned at a certain depth, and the rest of the network could work on data without any variability stemming from differences in object pose. 
 
 Note that the question of \emph{how to know} which transformation to apply for each image, something which is e.g. learned from data for STNs, is not the topic here. We simply show that even with perfect information about the pose of the input image,  invariance cannot be achieved by a spatial transformation of the feature map. 
\subsection{Intuition}
The key intuitions why a spatial transformation of CNN feature maps cannot, in the general case, align feature maps of a transformed image with those of an original image, and thus not enable invariant recognition, are as follows:
\begin{enumerate}[(i)]
	\item 
The natural way to align the feature maps of a transformed image with those of its original would be to
apply \emph{the inverse spatial transformation} to the feature maps of the transformed image i.e.
\begin{equation} 
\OTh^{-1}(\Lai \OTh f)_c(x) = (\Lai \OTh f)_c(\Th x).
\label{eq:inverse_alignment}
\end{equation} 
For example, to align the feature maps of an original and a rescaled image, we would, after feature extraction, apply the inverse scaling to the feature maps. 
We will show that using $\OTh^{-1}$ is, in fact, \emph{a necessary condition} for (\ref{eq:feature_alignment}) to hold. The reason for this is that the features for corresponding spatial positions after alignment will otherwise be computed from not fully overlapping image regions in the original image, in which case the output can clearly not be guaranteed to be equal.

	\item When transforming an input image, this typically causes not only \emph{a spatial shift} in its feature map representation but also \emph{a shift in the channel dimension} of the feature maps.
	This is illustrated in Figure~\ref{fig:tiny-proof} for the case of rotations, but a similar reasoning holds for a large range of spatial transformations. A \emph{purely spatial transformation} of the feature maps cannot correct for a change in e.g. which channels respond most strongly at a specific spatial position. Thus, a spatial transformation is not enough to align the feature maps of a transformed image with those of its original.

\item \emph{The receptive fields}, i.e. the region in the input that influence the response, of the features extracted in a neural network
(for a single layer, this corresponds to the support of the convolutional filters) are typically \emph{not invariant} to the relevant transformation group. Indeed, any finite support region will not be invariant to shears or transformations that contain a uniform or non-uniform scaling component. For example, for a scaling transformation, a filter applied to a rescaled image, might never cover the full object of interest, and thus the feature response cannot be guaranteed to take even \emph{the same set of values} for a rescaled image and its original. This is illustrated in Figure~\ref{fig:tiny-proof-scale}. 


\end{enumerate}
Since a purely spatial transformation cannot align the feature maps of a transformed image with those of its original, spatially transforming feature maps will not enable invariant recognition. The exception is if the features in the specific network layer are \emph{themselves invariant} to the relevant transformation. An example of this would be a network built from rotation invariant filters $\la$, where $\la(x) = \la(\Th x)$ for all $\la$. For such a network, or a network with more complex (learned or hardcoded) rotation invariant features in a certain layer, invariant recognition could be enabled by spatial transformations of the feature maps. 

One might, however, note that such invariant features in intermediate layers are in many cases not desirable (especially not early in the network), since they discard too much information about object pose. For example, rotation invariant edge detectors would lose  information about the edge orientations which tend to be important for subsequent tasks.

 




\subsection{Outline of proof}

\subsubsection{Single-layer case}
We first consider the case of \emph{a single convolutional layer} and show that the requirement that it should be possible to align feature maps implies very strict conditions on the filters. Lemma \ref{lemma:equiv1} shows that \emph{inversely transforming the feature maps} of a transformed image is equivalent to applying \emph{transformed filters} to the original image:
	$$  
		\OTh^{-1}\La_\la \OTh f =\La_{(\det \Th)  \OTh^{-1} (\la)}f.
	$$
	Lemma \ref{lemma:noRot} is the key to seeing that  $\OTh^{-1}$ is the only possible candidate to align the feature maps of a transformed image with those of its original, since otherwise features at corresponding spatial positions are computed from different parts of the original image. Finally, we discuss the conditions on the filters under which invariance is possible, where Lemma \ref{id} implies that we can give quite detailed conditions on the filters and transformations, since it says that if two single-layer networks compute the same function they must have the same filter/filters.

\subsubsection{Multi-layer case}
We then consider a more general non-linear feature extractor such as the \emph{multi-layer convolutional network} defined in Section \ref{sec:CNN} and show that similar strict conditions hold in this case. We first isolate two key features shared by single convolution operators and CNNs: \emph{translation covariance} and \emph{semi-locality}. These features underpin most of the proofs for the single-layer case and allow these proofs to be extended to the multi-layer case. 
\emph{Semi-locality} (Definition \ref{def:semi-local}) is an extension of the concept of an operator with \emph{compact support}. 
The reason to define the concept of semi-locality, instead of considering operators with compact support, is that we wish to include operators that output a constant for the input $f=0$, such as CNNs with non-zero biases or non-linearities that do not take zero to zero (or both) would do. 
We then show that the multi-layer continuous neural network (\ref{eq:CNN}) is a translation-covariant, semi-local operator.

Since it is not possible to give explicit conditions for individual filters (e.g. symmetries implies that the same function can be implemented by more than one set of filters), we will instead consider conditions that need to hold for the non-linear features extracted in a specific network layer $\Lai$, to enable aligning CNN feature maps of a transformed image with those of an original image at depth $i$.

A key step in our proof is to note that any translation-covariant operator $\La$ is captured by a map $\mu_\La: V \to \R$ defined by (equation (\ref{eq:mu-def}))
$$
\mu_\La(f):=(\La f) (\vec{0}),
$$
which we refer to as \emph{the generator}. The generator can be seen as a non-linear analog of a convolutional filter (evaluated at the origin for a single-layer network). Lemma \ref{lem:conj} and Lemma \ref{multi-equiv} then establish the relationship between the inversely transformed feature maps of a transformed image and the feature maps of the original image, showing that 
$$ \OTh^{-1}\La \OTh =\La f$$  
implies that $\mu_\La(\OTh f ) = \mu_\La(f)$. That  is, the network features must themselves \emph{already be invariant} to the relevant image transformation. 
Lemma \ref{lemma:noRot-multi} shows that, as for the single-layer case, $\OTh^{-1}$ is the only possible candidate to align the feature maps of a transformed image with those of its original.


\section{Covariance and invariance in the single-layer case}\label{sec:single}
Consider a single channel convolutional neural network with the filter $\la$
\begin{equation}\label{eq:single-layer-cnn}
\La_\la f(x) \defeq  (f\star \la)(x)=\int_{\R^N} f(y)\la(x-y) dy=\int_{\R^N} \la(y) f(x-y) dy.
\end{equation}
Can precomposing with $\OTh$ be undone after the convolution step by postcomposing with some other $\OTg$:
\begin{equation}\label{eq:alignment}
\OTg \La_\la \OTh  \stackrel{?}{=} \La_\la.
\end{equation} We will see that this is not possible. Note that since a spatial transformation of feature maps never mixes information between different channels, it is enough to show this for a network with a single feature channel.



\subsection{Covariance relations of convolution operators}

We begin by showing the following lemma, expressing naturality of convolution.

\begin{lem}\label{lemma:equiv1}
	\begin{equation}\label{eq:equiv1}  
	\OTh^{-1}\La_\la \OTh =\La_{(\det \Th)  \OTh^{-1} (\la)} \end{equation}

\end{lem}

\begin{proof}
	We compute using change of variables $u=\Th^{-1} y$, $du=(\det \Th^{-1}) \, dy$
	\begin{align}
	(\La_\la \OTh f)(x) &= \int_{\R^N} f(\Th^{-1} y)\la((x-y)) dy \nonumber \\
	&= \int_{\R^N} f(u)\la(\Th(\Th^{-1}x-\Th^{-1}y)) \det \Th du \nonumber\\
	&= \int_{\R^N} f(u)\la(\Th(\Th^{-1}x-u)) \det \Th du \nonumber\\
	&=  (\La_{(\det \Th) \OTh^{-1} \la} f) (\Th^{-1}x) \nonumber\\
	&=  ( \OTh \La_{(\det \Th) \OTh^{-1} \la} f) (x)
	\end{align}

	Applying $\OTh^{-1}$ to both sides yields the lemma.
	
\end{proof}

 Thus, inversely transforming \emph{the feature maps} of a transformed image will not yield the same feature maps as for the original image. Instead, this is equivalent to extracting features from the original image with \emph{transformed filters}.



\subsection{Using $\OTg=\OTh^{-1}$ is a necessary condition to align feature maps}

The following lemma will be the key to seeing that a necessary condition for being able to align the feature maps of a transformed image with those of its original is using $\OTh^{-1}$.  

\begin{lem}\label{lemma:noRot} 
If for two compactly supported filters $\la_1\neq 0$ and $\la_2\neq 0$ we have 
	$\La_{\la_1}=\OTh  \La_{\la_2}$ then $\Th=\Id$.
	
\end{lem}


%
%
%


\begin{proof}
	
	Since $\la_1\neq 0$, we can pick a compactly supported $f$ such that $(\La_{\la_1} f) (\vec{0})\neq 0$ (pick any $f$ with $\La_{\la_1} f\neq 0$, translate it to make $\La_{\la_1} f(\vec{0})\neq 0$, and, if needed, multiply by a bump function of sufficiently large ball to make it compactly supported). Suppose $f$ is supported on a ball of radius $r(f)$ around the origin and $\la_2$ on a ball of radius $r(\la_2)$ around the origin.   If $T_h\neq \Id$ we can pick $p$ such that $|T^{-1}_h(p)-p| >r(f)+ r(\la_2)+1$. Let $\hat{f}(x)=f(x+p)$ i.e. $\hat{f}=\calD_{-p}f$. Then using Lemma \ref{lemma:conv-trans-covar} we have	
	\begin{equation}
	(\calD_p \La_{\la_1} \hat{f} )(\vec{0})=(\La_{\la_1}   \calD_p  \hat{f} )(\vec{0}) =(\La_{\la_1}  f )(\vec{0})  \neq 0
	\end{equation}
	but 
	\begin{align}
	(\calD_p \OTh  \La_{\la_2} \hat{f} )(\vec{0})=( \OTh \calD_{T^{-1}_h p} \La_{\la_2} \hat{f})(\vec{0})=( \OTh  \La_{\la_2}   \calD_{T^{-1}_h p -p} f)(\vec{0})	 \nonumber \\
    =(\La_{\la_2}   \calD_{T^{-1}_h p -p} f) (T_h^{-1}(\vec{0}))=(\La_{\la_2}   \calD_{T^{-1}_h p -p} f)(\vec{0})=0,
	\end{align}	
	where the first equality follows from Lemma \ref{lem:conv-lin}, the second from Lemma \ref{lemma:conv-trans-covar}, and the last from the fact that  
	$\tilde{f}=\calD_{T^{-1}_h p -p} f$ is supported on a ball of radius $r(f)$ around $T^{-1}_h p -p$, which is disjoint from the ball of radius $r(\la_2)$ around the origin on which $\la_2$ is supported;  this means that in the convolution integral $( \La_{\la_2}\tilde{f})(0)= \int \tilde{f}(y) \lambda (-y) dy$  the integrand is zero at every point $y$, thus yielding the zero result, as wanted. 
\end{proof}


\subsection{Convolution determines the filter}

We now show that if two single-layer networks compute the same function, their filters must be equal. 

\begin{lem}\label{id}
	If $\La_{\la_1}=\La_{\la_2}$ then $\la_1=\la_2$.
\end{lem}
\begin{proof}
	
	Letting $\la=\la_1-\la_2$, we just need to show that $\La_\la=0$ implies $\la=0$. 
	
	
	
	Let   $f_n$  be a sequence of mollifiers converging to the delta function at the origin (that is a sequence of non-negative smooth functions each with integral equal to 1 and with their supports on balls of radii converging to $0$). Then (see for example  \cite{stein2009real}, Chapter 3, Theorem 2.3) we have $\La_\la f_n \to \la$ (in $L^1$), so that if   $\La_\la$ is the zero functional, then $\la$ is zero.
\end{proof}
This lemma implies that we can give more specific conditions on the filters in a single-layer network for which it is possible to achieve invariance by aligning CNN feature maps.

\subsection{Conclusions in the single-layer case}

 We can now conclude that the only admissible operator to align CNN feature maps is $\OTh^{-1}$ and that alignment is only possible if the convolutional filters are themselves invariant to the relevant transformation:
 
 \begin{pro}

If $ \OTg \La_\la \OTh =\La_\la$, this implies that $\OTg = \OTh^{-1} $ and that $\la= (\det \Th)  \OTh^{-1} (\la)$
\end{pro}

\begin{proof}
	Writing $\OTg=\OTH (\OTh)^{-1}$ and $\la_h=(\det \Th)  \OTh^{-1} (\la)$, we see that
	\begin{equation} 
	\OTg \La_\la \OTh=\OTH (\OTh)^{-1}\La_\la \OTh=\OTH \La_{\la_h}.
	\end{equation}

Thus, if (\ref{eq:alignment}) holds, by Lemma \ref{lemma:noRot} we must have $T_H=\Id$ and $\OTg =\OTh^{-1} $. Then, by Lemma~\ref{lemma:equiv1} and  Lemma~\ref{id} we must have 

\begin{equation}
\label{eq:equivI}\la= (\det \Th)  \OTh^{-1} (\la).
\end{equation}
	\end{proof}

This means that up to rescaling by $\det \Th$, the filter $\la$ is \emph{invariant} under the linear transformations $\Th$. Observe that this implies that $\la$ is invariant under all integer powers of $\Th$. If we further wish to have a network invariant to all transformations in a group $H$, then this also needs to hold for all $h \in H$. 

\begin{pro}
	The equality (\ref{eq:equivI}) is impossible for bounded non-zero $\la$ unless ${|\det \Th|=1}$. 
\end{pro}

\begin{proof}
	We have $|\sup  (\det \Th)  \OTh^{-1} (\la)|=|\det \Th||\sup \la|$, so if $|\sup \la|\neq 0$, $|\sup \la|\neq \infty$ and (\ref{eq:equivI}) holds then we must have $|\det \Th|=1$.
\end{proof}

One may be prepared to ignore intensity (aka rescaling), instead considering
\begin{equation}\label{eq:equivII}\la= C \OTh^{-1} (\la) \end{equation}
for some $C\in \R$. Even with this relaxation, this invariance can only hold for severely limited kinds of filters and transformations:

\begin{pro}
The equality (\ref{eq:equivII}) is impossible for $\la$ with support on a set of finite but non-zero measure, unless ${|\det \Th|=1}$.
\end{pro}

\begin{proof}
	 If $\la$ has support of measure $m$, then $\OTh^{-1} (\la)$ has support of measure $|\det \Th^{-1}|m$. If  (\ref{eq:equivII}) holds then $|\det \Th^{-1}|m=m$ and so if $m$ is finite and non-zero we must have $|\det \Th^{-1}|=1$, i.e. $|\det \Th|=1$.
\end{proof}


More strongly, in the case when the image domain is $\R^2$, one can use the classification of 2D real matrices by Jordan canonical form to study the behavior of iterations of $\Th$, as done, for example, in Chapter 3.1 of  \cite{hasselblatt2003first} (a very similar analysis is possible in higher dimensions).  Using this, we can analyze further even the cases where $|\det \Th|=1$, as follows.
	
	\begin{pro}\label{pro:single-layer}
		The equality (\ref{eq:equivII}) can hold  for $\la$ with support on a set of finite but non-zero measure only if $\Th$ is conjugate to some rotation or, if $T_h$ is orientation reversing, a reflection matrix; and in those cases only if (i) $\Th^n=Id$  for some $n$ and $\la$ is symmetric with respect to this finite set of transforms, or (ii) if $\la$ is constant on a collection of concentric ellipses along which $\Th$ rotates things.
	\end{pro}
	
	\begin{proof}
	

	There are special cases when all the  eigenvalues of $\Th$ are real and have absolute value 1. Then, either $\Th^2=\Id$, in which case $\lambda$ simply has to have a 2-fold symmetry (this includes the cases when $\Th$ is the reflection around the origin or a reflection through a line); or $\Th$ has Jordan form $\begin{pmatrix} 1&1\\0&1\end{pmatrix}$ or $\begin{pmatrix} -1&1\\0&-1\end{pmatrix}$ and $\Th^n=B^{-1}\begin{pmatrix} 1&n\\0&1\end{pmatrix}B$ or $\Th^n=B^{-1}\begin{pmatrix} (-1)^n&n\\0&(-1)^n\end{pmatrix}B$, respectively, for some fixed basis change matrix $B$. We see that the eigenspace of eigenvalue 1 is fixed, but everything else moves out to infinity, so an invariant $\lambda$ would have to be supported on this (1D) eigenspace (which would imply that the only possible invariant filter corresponds to a  $\La_\la$ which is zero).

	Similarly, if $|\det \Th|=1$ but $\Th$ has distinct real eigenvalues $d_1, d_2$ (this happens precisely when $\tr^2 \Th-4 \det \Th >0$), of size not equal to 1, $|d_1|>1>|d_2|$,  
	then $\Th$ has Jordan form $\begin{pmatrix} d_1&0\\0&d_2\end{pmatrix}$  and $\Th^n=B^{-1}\begin{pmatrix} d_1^n&0\\0&d_2^n\   \end{pmatrix}B$  for some fixed basis change matrix $B$;  everything  not in the $d_2$ eigenspace  moves out to infinity under positive iterations and everything not in $d_1$ eigenspace  under negative ones (in the new coordinates the motion is along hyperbolas $y=1/x$, and this is why such $\Th$ is called \emph{hyperbolic}), so an invariant $\lambda$ would have to be supported only at the origin.
	
	Further, \emph{the only remaining case} $|\det \Th|=1$ but $\tr^2 \Th-4 \det \Th <0$ (a.k.a. $\det \Th=1$, but $|\tr \Th| <2$), gives, up to a change of basis, \emph{a rotation matrix}. In the new basis, concentric circles around the origin are preserved by the rotation; in the original basis these are ``concentric" ellipses (this is the reason $\Th$ is called \emph{elliptic} in this case). If the rotation is by an irrational multiple of $\pi$, the orbit of any point is dense in the corresponding ellipse (see, for example, \cite{hasselblatt2003first}, Proposition 4.1.1) and equality (\ref{eq:equivII}) would still imply that $\la$ is constant on each of these ellipses. On the other hand, the $\Th$s where rotation is by a rational multiple of $\pi$ are precisely ones with $\Th^n=Id$ for some $n$.
\end{proof}	

Thus, we conclude that for a single-layer network, aligning the feature maps of a transformed image with those of its original is only possible for transformations that \emph{correspond to rotations or reflections} in some basis, and in that case only if the filters are themselves rotation/reflection invariant. Notably, such alignment is not possible for general affine transformations, scaling transformations or shears since there do not exist any non-trivial affine-, scale- or shear-invariant filters with compact support.


\section{Covariance and invariance in the multi-layer case}
\label{sec:multilayer}
We now give an equivalent proof for a more general non-linear, semi-local, translation-covariant feature extractor $\La$ (semi-locality is defined below). We are specifically interested in continuous multi-layer CNNs (Section \ref{sec:CNN}) but the proof is valid for any such operator. We ask whether equation (\ref{eq:alignment})
\begin{equation*}
\OTg^k \La  \OTh  \stackrel{?}{=} \La
\label{eq:post-pre-compose}
\end{equation*}	
could be true for such operators and if so under what conditions. 
Note that for the case of a multi-layer convolutional neural network, it is enough to consider a single feature channel at a certain depth, since a spatial transformation never mixes information between the channels. For simplicity, we will refer to a feature map at depth $i$ $(\Lai f)_c$ as $\La f$.

Two key features are shared by single convolution operators and CNNs: translation covariance and semi-locality. These features underpin most of the proofs for the single-layer case and allow these proofs to be extended to the multi-layer case.



\subsection{Commutators and conjugation of translation-covariant operators}
Recall that by Proposition \ref{prop:CNN-covar} the multi-layer CNN is a translation-covariant operator. We further note that translation covariance holds also when one changes coordinates on both input and output using $T_h$, i.e. when conjugating $\Lambda$ with the operator $\OTh$. 

\begin{lem}
	If $\La$ is translation covariant, then so is 
		$\OTh^{-1} \La \OTh$.
\end{lem}

\begin{proof} Using Lemma \ref{lem:conv-lin} and Definition \ref{def:translation-covariance} we compute:
	
	\begin{align}
	\calD_{x} \OTh^{-1} \La \OTh =&\nonumber\\
	=& \OTh^{-1} \calD_{T_h x} \La \OTh =\OTh^{-1}  \La \calD_{T_h x} \OTh =\OTh^{-1}  \La  \OTh \calD_{T_h^{-1}(T_h  x)}=\nonumber\\&\hspace{6.5cm} =\OTh^{-1}  \La  \OTh \calD_{ x}
	\end{align}
\end{proof}

\subsection{Generators of translation-covariant operators}

A key step in the multi-layer proof is to note that any translation-covariant operator $\La:V \to V$ is captured by a map $\mu_\La: V \to \R$ defined by 
\begin{equation}
\mu_\La(f):=(\La f) (\vec{0})
\label{eq:mu-def}.
\end{equation}
We call this $\mu_\La$  \textbf{the generator} of $\La$ (sometimes denoted simply by $\mu$ when the relevant $\La$ is clear from the context). 
Since we have
\begin{equation} 
(\La f) (x)=(\calD_{-x} \La f)  (\vec{0})=(\La \calD_{-x} f)(\vec{0})=\mu (\calD_{-x} f)
\label{eq:lambda-nonlinear-def},
\end{equation}
we can, conversely, given $\mu$ define a translation-covariant operator $\La_\mu$ by 
\begin{equation}
(\La_\mu f)(x) := \mu (\calD_{-x} f).
\label{eq:lambda-mu-def}
\end{equation}
Clearly the operations in (\ref{eq:mu-def}) and (\ref{eq:lambda-mu-def}) are inverses of each other. 
The generator $\mu_\La$ can be seen as a non-linear analog of a convolutional filter in the single-layer case.

The following Lemma is the equivalent to Lemma \ref{lemma:equiv1} in the single-layer case. 

\begin{lem}\label{lem:conj}
	The generator of $\OTh^{-1} \La_\mu \OTh$ is  $\mu_h:=\mu(\OTh f)$.
\end{lem}

\begin{proof}
	\begin{align}
	& (\OTh^{-1} \La_\mu \OTh f) (\vec{0}) = \text{\{definition of $\OTh$ (\ref{eg:Th-def})\}} \nonumber \\
	& = (\La_\mu \OTh f) (\Th \vec{0})\nonumber \\ 
	& = (\La_\mu \OTh f) (\vec{0}) = \text{\{definition of $\La_\mu$ (\ref{eq:lambda-mu-def}) \}} \nonumber\\
	& = \mu( \OTh f)
	\label{eq:lambda-postcomposed}
	\end{align}
\end{proof}
Thus, also in the case of a non-linear, translation-covariant feature extractor, inversely transforming \emph{the feature maps} of a transformed image will not yield the same feature maps as for the original image. Instead, it is corresponds to extracting features from transformed image patches.

\subsection{Semi-locality}
To enable considering operators that output a constant for the input $f=0$, we define the concept of \emph{semi-locality}. A semi-local operator is an extension of the concept of an operator with \emph{compact support}. It similarly implies that the output will only be affected by the values in a bounded region of the input image. However, that output does not necessarily have to be 0 for the input $f=0$ (but translation covariance implies that it must output a constant). 

\begin{defi}\label{def:semi-local}
	We will say that $\La$ is  \textbf{semi-local} if there exists a radius $r(\La)$ such that for any point $p$ and any two functions $f_1$ and $f_2$ which agree on the ball of radius $r(\La)$ around a point $p$  we have $\La f_1 (p)= \La f_2 (p)$. 
\end{defi}

Semi-locality interacts well with translation covariance. 
\begin{lem}\label{lem:semiloc-transl}
	If $\La$ is translation covariant and semi-local with radius $r(\La)$ and $f_1$ and $f_2$ agree on a ball of radius $r+r(\La)$ around $p$, then $\La f_1$ and $\La f_2$ agree on a ball of radius $r$ around $p$.
\end{lem}

\begin{proof}
	For any $x$ a in ball of radius $r$ around the origin, the functions $\calD_x f_1$ and $\calD_x f_2$ agree on a ball of radius $r(\La)$ around $p$; by definition of semi-locality, this means $(\La D_x f_1)(p)=(\La \calD_x f_2) (p)$, or $(\La  f_1)(p-x)=(\La  f_2) (p-x)$, which is what we wanted.
\end{proof}



Semi-locality is unaffected by conjugation with $\OTh$.

\begin{lem}\label{lemma:comp-loc}
	If $\La$ is semi-local, then so is $\OTh^{-1} \La \OTh$.
\end{lem} 
\begin{proof}Let  
$k=\max_{|v|=1} |\Th^{-1}(v)|$ be the operator norm of $\Th^{-1}$. Set $r=k r(\La)$. We claim $\OTh^{-1} \La \OTh$ is semilocal with radius $r$. Indeed, if $f_1$ and $f_2$ agree on a ball of radius $r$ around $p$, then $\OTh f_1$ and $\OTh f_2$ agree on ball of radius $r(\La)$ around $T_h p$, and so do the values $ (\La \OTh f_1)(T_h p)$ and $ (\La \OTh f_2)(T_h p)$ agree. This means $\OTh^{-1} (\La \OTh f_1)( p)=\OTh^{-1} (\La \OTh f_2)( p)$ as wanted.
	
\end{proof}

Convolutions with compactly-supported $\la$ are semi-local. 

\begin{lem}\label{lemma:conv-loc}
	If $\la$ is supported on a ball of radius $r(\la)$ around the origin, then $\La_\la$ is semi-local with radius $r(\la)$.
\end{lem}

\begin{proof} If $\la$ is supported on a ball $B$ of radius $r(\la)$ then we have
	
	$$\La f (p)=\int \f(p-y) \la(y) dy= \int_{B} \f(p-y) \la(y) dy.$$
	Thus, if $f_1$ and $f_2$ agree on the ball of radius $r(\la)$ around $p$, then the integrals for $f_1$ and $f_2$ agree, i.e. $\La f_1 (p)=\La f_2 (p)$.
\end{proof}

This simple Lemma \ref{lemma:conv-loc} is the basis of the following proposition.

\begin{pro} \label{prop:CNN-loc}
	A CNN as defined in Section \ref{sec:CNN} is a semi-local operator.
\end{pro}

\begin{proof} [Sketch]
	Observe that if two functions  agree on a ball of radius $R$, then after convolution with a kernel supported on a ball of radius $r$ the results agree at least on a ball of radius $R-r$. Applying a pointwise non-linearity $\sigma$ to each of the values does not affect this equality. Thus, if the radius $R$ is large enough, then after multiple convolution layers, the results are guaranteed to agree on some non-empty ball, which is what we wanted to prove.
	A more detailed proof (using induction and Lemmas  \ref{lem:semiloc-transl} and \ref{lemma:conv-loc}) is given in Appendix \ref{app:prop-CNN-proof}.
	
\end{proof}

\subsection{Covariance of the operator in the non-linear case}
We, now consider the conditions on $\mu$ or $f$ that are required for it to be possible to undo a precomposing with $\OTh$ after feature extraction  by postcomposing with $\OTh^{-1}$.

\begin{lem}\label{multi-equiv}
	Recall from Lemma \ref{lem:conj} that  $\mu_h (f)=\mu (\OTh f)$.	Then, for a general non-linear translation-covariant feature extractor $\La_\mu$ generated by $\mu$ (\ref{eq:lambda-mu-def}) 
	\begin{equation}
	(\OTh^{-1} \La_\mu \OTh f)  =  (\La_{\mu} f)   
	\end{equation}	
	implies
	\begin{equation}
	\mu = \mu_h,
	\end{equation}	
	i.e. that $\mu$ must be invariant to $\OTh$.
\end{lem}

\begin{proof}
	This is immediate from Lemma \ref{lem:conj}.
\end{proof}

Thus, for an inverse spatial transformation of the feature maps of a transformed image to render the same feature maps as for the original image, either $f$ must be invariant to $\OTh$ around every image point (which implies $f$ is constant) or
the feature extractor (i.e. the generator) must be invariant to the relevant transformation group.

\begin{defi}
We say that a functional $\La$ is \textbf{non-constant} if there exists an $f$ such that $\La(f)\neq \La (0)$.

\end{defi}

Observe that if the functional is semi-local, we can take $f$ to be compactly supported. A translation-covariant $\La$ is non-constant precisely when its generator $\mu$ is non-constant, i.e. there exists $f$ such that $\mu(f)\neq \mu (0)$ (Proof: take $f$ given by non-constancy of $\La$; then there is some $x$ such that $(\La f) (x) \neq (\La 0) (x) $, and $\hat{f}=\calD_x f$ has $\mu(\hat{f})\neq \mu (0)$).

\subsection{Using $\OTg=\OTh^{-1}$ is still a necessary condition to align feature maps}

The following lemma is the key to seeing that also in the non-linear case, a necessary condition for being able to align the feature maps of a transformed image with those of it’s original is using $\OTh^{-1}$. It is equivalent to Lemma \ref{lemma:noRot} in the single-layer case.

\begin{lem}
	\label{lemma:noRot-multi}
	If for two semi-local translation-covariant non-constant operators we have 
	$\La_{\mu_1}=\OTh  \La_{\mu_2}$    then $\OTh=\Id$. 
\end{lem}

\begin{proof}

	
	This is a more abstract version of the proof of Lemma \ref{lemma:noRot}. 
	First of all, applying $\La_{\mu_1}=\OTh  \La_{\mu_2}$ to the zero function we  get $\La_{\mu_1}0=\OTh  \La_{\mu_2} 0=\La_{\mu_2} 0$, and evaluating at location $\vec{0}$ obtain $\mu_1(0)=\mu_2(0)$.
	
	Now, take compactly supported $f$ with $\mu_1(f)\neq \mu_1(0)$. Suppose $f$ is supported in a ball of radius $r(f)$. 

	If $T_h\neq Id$, we can pick $p$ such that $|T^{-1}_h(p)-p| >r(\La_{\mu_2})+r(f)+1 $ (where $r(\La_{\mu_2})$ is  as in Definition \ref{def:semi-local}).

	Then,  by  (\ref{eq:lambda-mu-def}) we have	
	\begin{align}( \calD_{p} \La_{\mu_1} \calD_{-p} f(x)) (\vec{0})&= (\La_{\mu_1}  \calD_{p} \calD_{-p} f (x))(\vec{0}) \nonumber\\ &= (\La_{\mu_1} f (x))(\vec{0})=  \mu_1 (f (x))\neq \mu_1(0)		\end{align}

	but 
	\begin{align}( \calD_{p} \OTh  \La_{\mu_2} \calD_{-p} f (x))(\vec{0})=&\nonumber\\ 
	=&( \OTh  \calD_{T_h^{-1} (p)} \La_{\mu_2}  \calD_{-p} f (x))( \vec{0}) = \nonumber\\
	=&( \OTh   \La_{\mu_2} \calD_{T_h^{-1} (p)-p}  f (x))(\vec{0}) =\nonumber\\
	=& ( \La_{\mu_2}  D_{T_h^{-1}(p)-p} f )((T_h^{-1}(\vec{0}) )=\nonumber\\
	=&( \La_{\mu_2}  D_{T_h^{-1}(p)-p} f )(\vec{0})=\nonumber\\
	&\hspace{2cm}=(\La_{\mu_2}0)(\vec{0})=\mu_2(0)=\mu_1(0),
	\end{align}
	where the third-to-last equality (to $(\La_{\mu_2}0)(\vec{0})$) holds for the following reason: since $f(x)$ is supported on ball of radius $r(f)$ around the origin,   $D_{T_h^{-1}(p)-p} f( x)$ is supported on a ball of radius $ r(f)$ around $T_h^{-1}(p)-p$ which is entirely outside the ball of radius $r(\La_{\mu_2})$ around the origin. This means $\La_{\mu_2}$ applied to $D_{T_h^{-1}(p)-p} f( x)$ evaluated at the origin is equal to $\La_{\mu_2} 0$ evaluated at the origin by Definition \ref{def:semi-local} of semi-locality.

\end{proof}

\subsection{Conclusions in the multi-layer case}
We can now conclude also for the non-linear case that the  only  admissible  operator  to  align feature  maps  is $\OTh^{-1}$ and for alignment to be possible the extracted non-linear features must themselves be invariant to the relevant transformation. 


 \begin{pro}
	If $ \OTg \La_\mu \OTh =\La_\mu $, this implies that $\OTg = \OTh^{-1} $ and that $\mu(\OTh f ) = \mu(f).$
\end{pro}
	
\begin{proof}
Writing $\OTg=\OTH (\OTh)^{-1}$ and $\mu_h=\mu(\OTh f)$, we, as for the single-layer case, see that
	\begin{equation}  \OTg \La_\mu \OTh=\OTH (\OTh)^{-1}\La_\mu \OTh=\OTH \La_{\mu_h}.\end{equation}	
Suppose we do have 
\begin{equation}  \OTg \La_\mu \OTh =\La_\mu.\end{equation}
Then, by Lemma \ref{lemma:noRot-multi} (which is applicable because of Lemma \ref{lemma:comp-loc}) we must have $T_H=\Id$ and $\OTg =\OTh^{-1} $. Further, by Lemma~\ref{multi-equiv} we must have 
\begin{equation}\label{eq:mu-invar}
\mu(\OTh f ) = \mu(f)   
\end{equation}  if the equality (\ref{eq:post-pre-compose}) should hold for all $f$.
\end{proof}
Thus, the combined non-linear transformation must be computed from \emph{transformation invariant} non-linear operators $\mu$.
Since it is not possible to give explicit conditions for individual filters (e.g. symmetries implies that the same function can be implemented by more than one set of filters), we will instead investigate under which conditions invariant non-linear features $\mu_\La$ (\ref{eq:mu-def}) exist. 


\begin{pro}\label{pro:no-inv-has-contracting-direction}
If not all eigenvalues (real or complex) of $\Th$ have absolute value equal to 1, then for a continuous, semi-local, translation-covariant operator $\La$, equation (\ref{eq:mu-invar}) implies $\mu(f)=\mu(0)$ i.e. that $\La$ is the trivial operator that outputs the same constant signal for all inputs.  
\end{pro}

\begin{proof}

    We consider the case in which $T_h$ has at least one eigenvalue of absolute value bigger than $1$ (i.e. $T_h^{-1}$ has at least one eigenvalue of absolute value less than $1$). The case in which $T_h$ has at least one eigenvalue of absolute value less than $1$ follows by noting that invariance with respect to $\OTh$ is the same as invariance with respect to $\OTh^{-1}$.

    First, observe that for a translation-covariant operator, continuity of $\La$ implies continuity of $\mu$. Now, let $\La$ be semi-local with radius $r(\La)$. Let $\chi$ be the characteristic function of the ball of radius $r(\La)$. Then 
    \begin{equation}\label{eq:mu-loc}
        \mu(g)=\mu (\chi g )
    \end{equation}
    for any $g$ in $V=L^{1}_{loc}$. 
	
	We now decompose $\mathbb{R}^n$ into generalized eigenspaces of $\Th^{-1}$, $\mathbb{R}^n=E^{+}\oplus E^{0} \oplus E^{-}$ as in Section 3.3.3 in  \cite{hasselblatt2003first}.
	The condition that at least one eigenvalue of $\Th^{-1}$ have absolute value less than 1 means that $E^{-}$ is non-trivial.  By Corollary 3.3.7 in \cite{hasselblatt2003first},  when restricted to a non-trivial subspace $E^{-}\subseteq \mathbb{R}^n$ the operator $\Th^{-1}$ is eventually contracting (see Definition 2.6.11 ibid.), so that by  Corollary 2.6.13  and Lemma 3.3.6 ibid. under the iterates of $\Th^{-1}$ all points of $E^{-}$ converge to the origin with exponential speed. This implies that the points of $\mathbb{R}^n$ converge to points in the proper subspace $S=E^{+}\oplus E^{0}$. 
	
	Now starting with any $f$  in $L^{1}_{loc}$, and denoting by $B$ the ball of radius  $r(\La)$ around the origin, the functions $f_n =\chi  \OTh^n (\chi f) $  will eventually have supports lying in arbitrarily small neigbourhood of $S\cap B$, i.e. on a set of arbitrarily small measure. 
	If $\chi f$ is bounded, this implies that $f_n$ converge to the zero function in $L^{1}_{loc}$. Then, by continuity of $\mu$, the values $\mu(f_n)$ converge to $\mu(0)$. On the other hand, by semi-locality (\ref{eq:mu-loc}) and invariance (\ref{eq:mu-invar}) we get
	
	\begin{equation}\label{eq:mu-const}
	\mu(f_n)=\mu(\chi  \OTh^n (\chi f))=\mu(  \OTh^n (\chi f))=\mu(\chi f)=\mu(f).
	\end{equation}
	
	We conclude $\mu(f)=\mu(0)$ for any $f$ in $L^{1}_{loc}$ with bounded $\chi f$. Since any $f$ in $L^{1}_{loc}$ can be approximated arbitrarily well by functions $g_i$ with  bounded $\chi g_i$, and $\mu$ is continuous, we conclude that $\mu(f)=\lim \mu(g_i)= \mu(0)$ for all $f$.

\end{proof}

In the 2D case we can enhance this further to give conclusions similar to those of Proposition \ref{pro:single-layer}.

\begin{pro}\label{pro:multilayer-2D}
		The equality (\ref{eq:mu-invar}) can hold  
		for a continuous, semi-local, translation-covariant operator $\La$ only if $\Th$ is conjugate to some rotation or, if $T_h$ is orientation reversing, a reflection matrix.
	\end{pro}

	\begin{proof}
	    As in the proof of Proposition \ref{pro:single-layer}, studying the Jordan form of $\Th$ shows that the only cases not covered by Proposition \ref{pro:no-inv-has-contracting-direction} are ones when $\Th$ is conjugate to $\begin{pmatrix}1&1\\0&1
	    \end{pmatrix}$ or $\begin{pmatrix}-1&1\\0&-1
	    \end{pmatrix}$ (this is the case of shear transformations). In this case $\Th$ does not have iterates that contract $\mathbb{R}^2$ to a proper subspace, but the intersection of images of $B$ under $\Th$ with $B$ still lie arbitrarily close to a 1-D subspace. Then the same proof as in Proposition \ref{pro:no-inv-has-contracting-direction} yields the result.
	\end{proof}

\begin{rem}
 In the higher dimensional case, one can perform very similar analysis based on Jordan form of $\Th$ and extend the proof of Proposition \ref{pro:multilayer-2D} to conclude that invariance with respect to $\Th$ can only be obtained if $\Th$ is conjugate to an orthogonal matrix.
\end{rem}
  Thus, we reach a very similar conclusion as for the single-layer case. To enable aligning feature maps of a transformed image with those of its original, the non-linear  features $\mu_\La$ 
 (\ref{eq:mu-def}) must be invariant to the relevant transformation. 
 Furthermore, Propositions \ref{pro:no-inv-has-contracting-direction} and \ref{pro:multilayer-2D} show that there \emph{does not exist} any such invariant non-linear features $\mu_\La$  
 unless $\Th$ corresponds to a rotation or a reflection (or in higher dimensions an orthogonal) matrix in some coordinate system. In other words, there does not exist any such features invariant to affine transformations, scaling transformations or shears.
Since the restricted covariance relation (\ref{eq:covariance2}) cannot hold for these transformations, purely spatial transformations of feature maps \emph{cannot enable affine- scale- or shear-invariant recognition}.
 These conclusions hold for any continuous, semi-local, translation-covariant operator, which in particular includes $\La$ given by a CNN (\ref{eq:CNN}) with Lipschitz continuous non-linearities $\sigma_i$. 







\section{Summary and conclusions}
Using elementary analysis, we have presented a proof that spatial transformations cannot, in general, align CNN feature maps of a transformed image to match those of its original.
We have showed that, in order for feature extraction and spatial transformations to commute for translation-covariant, semi-local operators (such as CNNs), the features computed by the network must themselves be \emph{invariant} to the relevant image transformation.
 Since this is not generally the case, applying the inverse spatial transformation to a feature map extracted from a transformed image will typically \emph{not render the same feature map} as for the original image. This can be contrasted with the case of pure translations, where the translation covariance of a CNN implies that a translation of the input indeed corresponds to a translation of the feature maps.
 
 Furthermore, we have shown that features computed with convolutional filters of compact support and Lipschitz continuous non-linearities (such as would be the case for a standard CNN) can only be made invariant to transformations that \emph{correspond to reflections or rotations} in some basis.
In other words, there does not exist any such features invariant to affine transformations, scaling transformations or shear transformations.
Thus, spatial transformations of feature maps \emph{cannot enable affine-, scale-, or shear-invariant recognition} for CNNs or indeed any continuous, semi-local, translation-covariant feature extractor.



Our results imply that methods based on spatial  transformations  of  CNN  feature  maps  or  filters (e.g. \cite{HeZhaXia-ECCV2014,yuarXiv2015,DaiQiXio-arXiv2017,JadSimZisKav-NIPS2015}) is not a replacement for image alignment of the input. In particular, transforming feature maps \emph{cannot enable invariant recognition for general affine transformations, scaling transformations or shear transformations}, and it will only enable rotation-invariant recognition for networks with learnt or hardcoded {rotation-invariant filters/features}. 

\begin{appendix}
	\section{Appendix}

	\subsection{Proof that a single convolutional layer is translation covariant}
	\label{app:single-layer-covariance}

	\begin{pro}
		A single-layer continuous CNN 
		 (\ref{eq:single-layer-cnn}) is translation covariant: 
	\begin{equation}
	    \calD_{\delta}\La_\la =\La_\la \calD_{\delta}=\La_{\calD_{\delta}\la}.
	\end{equation}
	\end{pro}
	
	\begin{proof} We compute
		
		\begin{align}(\calD_{\delta}\La_\la f)(x)=(\La_\la f)(x-\delta)=\int_{\R^N} f(y)\la(x-\delta-y)  dy=(\La_{D_{\delta}\la} f)(x)\end{align}
		
		and using the change of variables $u= y-\delta$
		
		\begin{align}
		(\La_\la \calD_{\delta} f)(x) &= \int_{\R^N} f( y-\delta)\la(x-y) dy = \int_{\R^N} f(u)\la(x-\delta-u)  du=\nonumber\\
		& \hspace{6cm}=(\La_{D_{\delta}\la} f)(x).
		\end{align}
		
	\end{proof}

	\subsection{Proof that CNNs are semi-local and translation covariant}
	\label{app:prop-CNN-proof}

	Recall Propositions  \ref{prop:CNN-covar} and \ref{prop:CNN-loc}:
	
	\begin{pro}
		A multi-layer continuous CNN, as defined in Section \ref{sec:CNN}, is a translation-covariant semi-local operator.
	\end{pro}

	\begin{proof}

		The proof is inductive and is based on  (\ref{eq:CNN}) which we copy here for convenience:
		
		\begin{multline}\label{eq:CNN-App}
		(\Lambda^{(i)} f)_c (x) =  \sigma_i \left( \sum_{m=1}^{M_{i-1}} \int_{y \in \R^N } (\Lambda^{(i-1)}f)_m (x-y)\, \lambda^{(i)}_{m,c}(y) \, dy + b_{i,c}
		\right) 
		\end{multline}
		
		We will prove that $ (\Lambda^{(i)} f)_c$ in (\ref{eq:CNN-App}) are translation covariant and  semi-local  by induction on $i$.
		The base case when $i=0$ and $\Lambda^{(i)} f=f$ is immediate.
		The induction step for translation covariance is immediate from the  formula (\ref{eq:CNN-App}) and the fact that a single convolution is translation covariant (Lemma \ref{lemma:conv-trans-covar}).
		
		For semi-locality, denoting, as before, for any convolution kernel $\lambda$ by $r(\lambda)$ radius such that $\lambda$ is supported on a ball of radius $r(\lambda)$, we pick  
		\begin{equation}r(\La^{i}_c) = \max_{m} [r(\La^{(i-1)}_m)+r(\lambda^{(i)}_{m, c})].
		\end{equation}
		Observe that since by the induction hypothesis, $\La^{(i-1)}_m$ is semi-local with radius $r(\La^{(i-1)}_m)$, if $f_1$ and $f_2$ agree on a  ball of radius $[r(\La^{(i-1)}_m)+r(\lambda^{(i)}_{m, c})]$ around some  $p$, then  by Lemma \ref{lem:semiloc-transl} the functions $(\Lambda^{(i-1)} f_1)_m(x)$  and  $(\Lambda^{(i-1)} f_2)_m(x)$  agree over the ball $B$ of radius   $r(\lambda^{(i-1)}_{m, c})$ around  $p$, and we denote this common function on the ball by $f^{i-1}_m$.
		By Lemma \ref{lemma:conv-loc} the convolution integrals for the specific $m$ in formula (\ref{eq:CNN}) for $f_1$ and $f_2$ evaluated at $p$  are equal.
		Therefore, if $f_1$ and $f_2$ agree on a ball of radius $r(\La^{i}_c)$ around $p$ then the overall expressions computed by formula (\ref{eq:CNN}) for $f_1$ and $f_2$ at $p$ will be equal, which is exactly what we set out to prove.
		
		Finally, the non-linearity $\sigma_i$ applies the same function to values at all locations so does not affect either translation covariance, nor semilocality (the equality $(\La f_1)(p) = (\La f_2)(p)$ is preserved when applying a pointwise non-linearity).
	\end{proof}
	
\end{appendix}

\bibliographystyle{splncs}
\bibliography{bib/yjdeepl,bib/stn_extra}

\begin{thebibliography}{10}

\bibitem{ChoGwaSavSil-NIPS2016}
Choy, C.B., Gwak, J., Savarese, S., Chandraker, M.:
\newblock Universal correspondence network.
\newblock In: Advances in Neural Information Processing Systems. (2016)
  2414--2422

\bibitem{LiCheCaiDav-arXiv2017}
Li, J., Chen, Y., Cai, L., Davidson, I., Ji, S.:
\newblock Dense transformer networks.
\newblock arXiv preprint arXiv:1705.08881 (2017)

\bibitem{KimLinJeoMin-NIPS2018}
Kim, S., Lin, S., JEON, S.R., Min, D., Sohn, K.:
\newblock Recurrent transformer networks for semantic correspondence.
\newblock In: Advances in Neural Information Processing Systems. (2018)
  6126--6136

\bibitem{zheng2018pedestrian}
Zheng, Z., Zheng, L., Yang, Y.:
\newblock Pedestrian alignment network for large-scale person
  re-identification.
\newblock IEEE Transactions on Circuits and Systems for Video Technology (2018)

\bibitem{HeZhaXia-ECCV2014}
He, K., Zhang, X., Ren, S., Sun, J.:
\newblock Spatial pyramid pooling in deep convolutional networks for visual
  recognition.
\newblock In: European Conference on Computer Vision, Springer (2014)  346--361

\bibitem{yuarXiv2015}
Yu, F., Koltun, V.:
\newblock Multi-scale context aggregation by dilated convolutions.
\newblock arXiv preprint arXiv:1511.07122 (2015)

\bibitem{DaiQiXio-arXiv2017}
Dai, J., Qi, H., Xiong, Y., Li, Y., Zhang, G., Hu, H., Wei, Y.:
\newblock Deformable convolutional networks.
\newblock CoRR, abs/1703.06211 \textbf{1} (2017) ~3

\bibitem{JadSimZisKav-NIPS2015}
Jaderberg, M., Simonyan, K., Zisserman, A., Kavukcuoglu, K.:
\newblock Spatial transformer networks.
\newblock In: Advances in Neural Information Processing Systems (NIPS). (2015)
  2017--2025

\bibitem{cohen2016group}
Cohen, T., Welling, M.:
\newblock Group equivariant convolutional networks.
\newblock In: International conference on machine learning. (2016)  2990--2999

\bibitem{CohGeiWei-NIPS2019}
Cohen, T.S., Geiger, M., Weiler, M.:
\newblock A general theory of equivariant {CNN}s on homogeneous spaces.
\newblock In: Advances in Neural Information Processing Systems. (2019)
  9142--9153

\bibitem{FinJanLin-arXiv2020}
Finnveden, L., Jansson, Y., Lindeberg, T.:
\newblock Understanding when spatial transformer networks do not support
  invariance, and what to do about it.
\newblock arXiv preprint arXiv:2004.11678 (2020)

\bibitem{stein2009real}
Stein, E.M., Shakarchi, R.:
\newblock Real analysis: measure theory, integration, and Hilbert spaces.
\newblock Princeton University Press (2009)

\bibitem{hasselblatt2003first}
Hasselblatt, B., Katok, A.:
\newblock A first course in dynamics: with a panorama of recent developments.
\newblock Cambridge University Press (2003)

\end{thebibliography}

\end{document}